\documentclass{article}
\usepackage{kr}

\usepackage{times}
\usepackage{url}
\usepackage[hidelinks]{hyperref}
\usepackage[utf8]{inputenc}
\usepackage[small]{caption}
\usepackage{graphicx}
\usepackage{amsmath}
\usepackage{amsfonts}
\usepackage{amsthm}
\usepackage{bm}
\usepackage{booktabs}
\usepackage{algorithm}
\usepackage{algorithmic}
\newtheorem{prop}{Proposition}

\title{Implementing Derivations of Definite Logic Programs\\ with Self-Attention Networks}

\author{
Phan Thi Thanh Thuy$^{1}$
\and 
Akihiro Yamamoto$^{1, 2}$
\affiliations
$^1$Graduate School of Informatics\\
$^2$Center for Innovative Research and Education in Data Science\\
Kyoto University\\
\emails
thuy@iip.ist.i.kyoto-u.ac.jp,
akihiro@i.kyoto-u.ac.jp
}

\begin{document}

\maketitle

\begin{abstract}
In this paper we propose that a restricted version of logical inference can be implemented with self-attention networks.
We are aiming at showing that LLMs (Large Language Models) constructed with
transformer networks can make logical inferences.  We would reveal the potential of LLMs
by analyzing self-attention networks, which are main components of transformer networks. 
Our approach is not based on semantics of natural languages but operations of logical inference. 
We show that
hierarchical constructions of self-attention networks with feed forward networks (FFNs)
can implement top-down derivations for a class of logical formulae. 
We also show bottom-up derivations are also implemented for the same class.
We believe that our results show that LLMs implicitly have the power of logical inference.
\end{abstract}

\section{Introduction}
Large Language Models (LLMs) are giving strong impacts to our everyday life. 
Many people begin to make use of them in various manners and to expect that more power is given to them. 
An example of such power is logical inference. Some say that LLMs can make
logical inference, and make discussions on the semantical correctness of outputs made by LLMs, where semantics are meanings of sentences in natural languages.
Referring the theory of mathematical logic, the correctness of logical inference 
should be supported not only semantical manners but also operational. 
Operations of logical inference are methods for deriving conclusions 
from assumptions and showing the truth of sentences based on them.
We take an operational approach
towards showing the potential of logical inference in LLMs.
More precisely we analyze the transformer networks, which are known as the fundamental mechanism of major LLMs, in particular, self-attention networks
which are main components of transformers~\cite{DBLP:conf/nips/VaswaniSPUJGKP17}.

As logical inference mechanisms we employ top-down derivations for definite logic programs and queries. 
On the relation between bottom-up derivations of 
definite logic programs and neural networks,
a stream of research starting with~\cite{DBLP:conf/ksem/SakamaIS17}
has been made
\cite{DBLP:conf/miwai/NguyenSSI18}
\cite{DBLP:conf/ilp/AspisBR18}
\cite{DBLP:journals/logcom/NguyenSSI21}
\cite{DBLP:journals/amai/SakamaIS21}.
Each of these presents a method to represent a definite logic program
with a matrix so that matrix multiplication plus some additional operation
corresponds to bottom-up derivation.
We first show that top-down derivations can be implemented by a type of self-attention networks.
Also we show that the bottom-up derivation treated in the previous research 
can be implemented by another type of self-attention networks.

A self-attention network takes three inputs: a vector representing a query, a matrix 
representing a set of keys, and a matrix representing a set of values.
These inputs remind us the operations of derivation made of a query, the head of a definite clause, and its body.
Our fundamental idea is to make correspondence between the inputs of self-attention networks and the three operations of derivation. 
We also employ the hardmax function instead of the softmax function
used in self-attention networks. This is from the previous research on the analysis of 
self-attention networks on the viewpoints of acceptors of formal languages
\cite{DBLP:journals/tacl/Hahn20}
\cite{DBLP:conf/acl/YaoPPN20}
\cite{DBLP:journals/jmlr/PerezBM21}.

\section{Preliminaries}
\subsection{Self-Attention Networks}
Following \cite{DBLP:conf/nips/VaswaniSPUJGKP17} the encoder part of a transformer is constructed of
a position encoder of the inputs and $N$ layers of neural networks following the position encoder.
Each layer consists of two components,  a self-attention network and a feed-forward network (FFN). 
In this paper we do not use the position encoder. 
The self-attention network in \cite{DBLP:conf/nips/VaswaniSPUJGKP17} takes a set of queries,
a set of keys, and a set of values as its inputs and outputs a new set of queries.
All queries, keys, and values are vectors. 
Following the original paper, let $\bm{q}^k$ be a {\em raw-vector} for a query, and 
let 
\[
K= 
\left ( 
\begin{array}{c}
\bm{k}_1\\
\bm{k}_2\\
\vdots\\
\bm{k}_{m}
\end{array}
\right )
\mbox{ and }
V= 
\left ( 
\begin{array}{c}
\bm{v}_1\\
\bm{v}_2\\
\vdots\\
\bm{v}_{m}
\end{array}
\right )
\] be respectively arrays for a set of keys and a set of values. 
The function of the $k$-th layer of a self-attention network is 
\begin{eqnarray}
\label{attention}
\mbox{Attention}(\bm{q}^k, K, V) &=& \mbox{softmax}\left ( \frac{\bm{q}^kK^{\top}}{\sqrt{d}} \right ) V,
\end{eqnarray}
where $d$ is the dimension of vectors for queries and keys.
In our discussion we omit the normalization with $\sqrt{d}$ and 
replace the softmax function with the hardmax function as in the analysis of 
self-attention networks with traditional theories of formal languages
\cite{DBLP:journals/tacl/Hahn20,DBLP:conf/acl/YaoPPN20}.
The hardmax function is defined as
\begin{eqnarray*}
\mbox{hardmax} (x_1,\ldots, x_m) &=& (y_1,\ldots, y_m), 
\end{eqnarray*}
where 
\begin{eqnarray*}
y_i &=& \left \{ 
\begin{array}{ll}
\frac{1}{M} & \mbox{ if $x_i$ is a maximum of $(x_1,\ldots, x_m)$},\\
0 & \mbox{ otherwise}
\end{array}
\right .
\end{eqnarray*}
and $M$ is the number of maximums in $(x_1,\ldots, x_m)$.
Then the self-attention function (\ref{attention}) is represented as
\begin{eqnarray*}
\bm{a}^k
&=&\mbox{Attention}(\bm{q}^k, K, V) =
\sum_{j=1}^{m}s_{j}\bm{v}_j,\,\,\, \mbox{ where} \\
&&(s_{1}, \ldots, s_{m}) = \mbox{hardmax}((\langle \bm{q}^k, \bm{k}_1\rangle,\ldots, \langle \bm{q}^k, \bm{k}_m\rangle)).
\end{eqnarray*}
Let the FFN following the self-attention implement a function $f$. 
The output of the $k$-th layer is 
\[
\bm{q}^{k+1}=f(\bm{a}^k),
\]
and this is passed to the $(k+1)$-th layer as its input.

\subsection{Definite Logic Programs and Derivations}
Every formula in porpositional logic consist of propositional variables and logical connectives.
In our case we use two logical connectives: $\wedge$ meaning ``AND'' and $\leftarrow$ meaning 
``IF''. 
For example, $p\leftarrow q\wedge r$ is interpreted as ``A proposition $p$ holds if both $q$ and $r$ hold.''
In this paper we call every propositional variable a {\em propositional symbol\/}, or simply, a {\em symbol.\/}

We give a simple example of logical formulae which we treat in our discussion. 
Let the set of propositional symbols be $p$, $q$, $r$, $s$, $t$, $u$, and $w$. 
We prepare a special symbol $\top$, which means  ``TRUE'', and
$\bot$, which means ``FALSE''. 
We say a logical formula is a {\em conjunction} or a {\em query} if it contains a single symbol or multiple symbols
connected with $\wedge$.  For example,
\[ 
p\wedge q \wedge r,\,\,\, p, \,\,\,\mbox{ and } \top 
\]
are conjunctions, and also called queries. 
We say a logical formula is a  {\em definite clause} if it contains a connective $\leftarrow$ with a single symbol 
on the left-hand side and a conjunction on the right-hand side.  For  example, 
\[
\begin{array}{l}
p \leftarrow q\wedge r\\
q \leftarrow s\\
r \leftarrow s\wedge t \\
s \leftarrow u\\
t \leftarrow \top\\
u \leftarrow \top\\
w \leftarrow \bot
\end{array}
\]
are definite clauses.
The lhs of a definite clause is 
called its {\em head} and the conjunction of its rhs is called its {\em body}.

We explain {\em top-down derivations} with a simple example. 
Let $P$ denote  the set of definite clauses above. 
A top-down derivation starts with a query.
Assume that a query consisting of one symbol $p$ is given. 
Then a definite clause in $P$ whose head matches with $p$ is searched.
In this case $p \leftarrow q \wedge r$ is found. Then $p$ in the query
is replaced with the body of the definite clause, and a goal clause
\[
q\wedge r
\]
is derived. Next a definite clause whose head matches $q$ and a clause whose head matches $r$
are searched, and a query
\[
s\wedge s \wedge t
\]
is derived. By the idempotent property of $\wedge$, the query is simplified into
\[
s \wedge t.
\]
In repetition of the same operation we eventually obtain a derivation sequence illustrated in Fig.~1.
The last query is simplified into $\top$, which means ``The first query $p$ is proved.''
\begin{figure}
\[
\begin{array}{c}
p\\
| \\
q\wedge r \\
| \\
s\wedge s \wedge t \equiv s \wedge t\\
|\\
u \wedge \top\\
|\\
\top \wedge \top \equiv \top
\end{array}
\]
\caption{A successful top-down derivation}
\end{figure}

We give formal definitions. 
Let $\Pi$ be a finite set of propositional symbols.
A {\em conjunction} is a formula of the form 
\begin{eqnarray}
\label{goal}
q_1\wedge\ldots \wedge q_n \hspace{1cm}(n\geq 1), 
\end{eqnarray}
where $q_1$, $\ldots$, $q_n$ $\in \Pi \cup \{ \top,\, \bot \}$.
A conjunction is also called a {\em query}.
Since we employing propositional logic, we assume that 
$q_1$,$\ldots$, and $q_n$ are mutually different.
A {\em defnite clause} $C$ is a formula of the form
\begin{eqnarray*}
p \leftarrow q_1\wedge q_2 \wedge \ldots \wedge q_n \hspace{1cm}(n\geq 1)
\end{eqnarray*}
such that $p$, $q_1$,$\ldots$, $q_n$ $\in \Pi$, 
\begin{eqnarray*}
p \leftarrow \top 
\end{eqnarray*}
or
\begin{eqnarray}
\label{fail}
p \leftarrow \bot
\end{eqnarray}
with $p\in \Pi$.
The proposition $p$ is called the {\em head} of the clause and
denoted by $\mathrm{head}(C)$. The conjunction 
$q_1\wedge q_2 \wedge \ldots \wedge q_n$, $\top$, or $\bot$ is called its {\em body}
and denoted by $\mathrm{body}(C)$.
A finite set of definite clauses is called a {\em definite program}.
We assume that every program must satisfy the restriction that no pair of definite clauses
in the program share their head. Such a program is called an {\em SD-program} \cite{DBLP:conf/miwai/NguyenSSI18}
\footnote{Precisely speaking,  formulae of the form (\ref{fail}) are not allowed in \cite{DBLP:conf/miwai/NguyenSSI18}}.

Given an SD-program $P$, a {\em one-step top-down derivation} of a 
conjunction as a query (\ref{goal}) 
is to find, for each $i$ $(0< i \leq m)$, a definite clause in $P$ of the form 
\begin{eqnarray*}
p_i\leftarrow q_{i1}\wedge\ldots\wedge q_{i n_i},
\end{eqnarray*}
to make a query
\begin{eqnarray*}
q_{11}\wedge\ldots\wedge q_{1 n_1}\wedge\ldots\wedge q_{m1}\wedge\ldots\wedge q_{m n_{i_m}},
\end{eqnarray*}
and to simplify this query
by duplicating propositional symbols into one so that
the query consists of mutually different symbols. 
A {\em top-down derivation} of a query is a finite repetition of  one-step derivations.
If a conjunction consisting only of $\top$ is obtained, we say that the top-down derivation is {\em successful}.
It is {\em failed\/} if a query contains $\bot$ is obtained. 

In the traditional framework of resolution principle \cite{chang-lee}\cite{lloyd}, 
making top-down derivations would be called linear resolution with negating the
query (\ref{goal}) and deriving ``contradiction''.
A problem appears in our discussion in the treatment of propositional symbols which do not appear any definite clauses.  To such a symbol appears in a query, the resolution operation cannot be applied but no method is provided to present the situation explicitly. 
In our discussion we would represent the failure of derivation with the clause of the form (\ref{fail}) and therefore would not be in the framework of resolution, but in the framework of top-down and backward proving of conjunctions.

Note that the previous papers listed in Section~1 treat bottom-up derivations
with the $T_P$ operator \cite{lloyd} and regard
every conjunction without $\top$ and $\bot$ an {\em interpretation}. 

\section{Implementing Top-Down Derivations with Self-Attention Networks}
First we give an illustration with the previous example. 
From each definite clause $C$ in $P$ we make two vectors $\bm{h}_C$ and $\bm{b}_C$
with a vector of 9 dimensions.
Each dimension 
corresponds to $p$, $q$, $r$, $s$, $t$, $u$, $w$, $\top$, and  $\bot$.
The vector $\bm{h}$ shows the symbol appearing in the head, and 
$\bm{b}$ appearing in the body. 
Clearly every $\bm{h}$ is a unit vector.
From the first definite clause in $P$ we get two vectors
\[
\begin{array}{lcll}
&&\ p\ \ q \ \ r \ \ s \ \ t \ \ u \ \ w \ \top\, \bot\\
\bm{h}_C&=&(1, 0, 0, 0, 0, 0, 0, 0, 0)& \mbox{and}\\
\bm{b}_C&=&(0, 1, 1, 0, 0, 0, 0, 0, 0).
\end{array}
\]
Since $P$ is an SD-program, we can indicate these vectors with  $p=\mathrm{head}(C)$ and
write $\bm{h}_p$ and $\bm{b}_p$ instead of $\bm{h}_C$ and $\bm{b}_C$.
Additionally we prepare two vectors for each of $\top$ and $\bot$ as follows:
\begin{eqnarray*}
&& \, p\ \ q \ \ r \ \ s \ \ t \ \ u \ \ w\,\top\, \bot\\
\bm{h}_\top & = & (0, 0, 0, 0, 0, 0, 0, 1, 0),\\
\bm{b}_\top & = & (0, 0, 0, 0, 0, 0, 0, 1, 0),\\
\bm{h}_\bot & = & (0, 0, 0, 0, 0, 0, 0, 0, 1), \mbox{ and}\\
\bm{b}_\bot & = & (0, 0, 0, 0, 0, 0, 0, 0, 1).
\end{eqnarray*}

From vectors $\bm{h}_p$, $\bm{h}_q$,$\ldots$, $\bm{h}_\top$ we get an identity matrix $I_9$
of 9-dimension.
In implementing a one-step top-down derivation with a self-attention network
we substitute $I_9$ to the key matrix $K$ in (\ref{attention}).
To the value matrix $V$ in (\ref{attention}), we substitute
\[
B_P = 
\begin{array}{c}
\begin{array}{c}
\\
\end{array}\\
\left ( 
\begin{array}{c}
\bm{b}_p\\
\bm{b}_q\\
\bm{b}_r\\
\bm{b}_s\\
\bm{b}_t\\
\bm{b}_u\\
\bm{b}_w\\
\bm{b}_\top\\
\bm{b}_\bot
\end{array}
\right )
\end{array}
=
\begin{array}{c}
\begin{array}{ccccccccc}
p& q & r& s & t & u& w & \top& \bot\\
\end{array}\\
\left (
\begin{array}{ccccccccc}
0 & 1 & 1 & 0 &0 & 0 & 0 & 0 & 0\\
0 & 0 & 0 & 1 & 0 & 0 & 0 &0 & 0 \\
0 & 0 & 0 & 1 & 1 & 0 & 0 &0 & 0\\
0 & 0 & 0 & 0 & 0 & 1 & 0 &0 & 0 \\
0 & 0& 0 & 0 & 0 & 0 & 0 & 1 & 0\\
0 & 0& 0 & 0 & 0 & 0 & 0 & 1 & 0\\
0 & 0& 0 & 0 & 0 & 0 & 0 & 0 & 1\\
0 & 0& 0 & 0 & 0 & 0 & 0 &1 & 0\\
0 & 0 & 0 & 0 & 0 & 0 & 0 & 0 & 1
\end{array}
\right )
\end{array}
\]
As the function implemented by the FFN following the self-attention network, we use 
the {\em  dimension-wise Heaviside function\/}
\begin{eqnarray*}
H((v_1, v_2, \ldots)) = (H_0(v_1), H_0(v_2), \dots),
\end{eqnarray*}
where $H_0$ is a Heaviside function defined as
\begin{eqnarray*}
H_0(x)& =& \left \{ \begin{array}{ll}
1 & \mbox{ if $x> 0$,}\\
0 & \mbox{ otherwise.}
\end{array} \right . 
\end{eqnarray*}

The top-down derivation illustrated in the last section is represented as follows.
First we represent the query consisting only of $p$ as a vector 
\[
\bm{q}^1 = (1, 0, 0, 0, 0, 0, 0, 0, 0)
\]
and  put this into the first layer as its input.
The computation of $(s_1,\ldots,s_9)$ is 
\[
\mathrm{hardmax}(\langle \bm{q}^1, \bm{k}_1\rangle,..., \langle \bm{q}^1, \bm{k}_9\rangle)=(1, 0, 0, 0, 0, 0, 0, 0, 0)
\]
and so the output of the self-attention network is 
\[
\bm{a}^1 = (0, 1, 1, 0, 0, 0, 0, 0, 0).
\]
Next we put $\bm{a}^1$ into the FFN following the self-attention network and put the result $\bm{q}^2 = H(\bm{a}^1)$
into the second layer as its input, which represents $q\wedge r$. Then
\[
\mathrm{hardmax}(\langle \bm{q}^2, \bm{k}_1\rangle,..., \langle \bm{q}^2, \bm{k}_9\rangle)=(0, \frac{\,1\,}{2}, \frac{\,1\, }{2}, 0, 0, 0, 0, 0, 0)
\]
and the result of the self-attention network is 
\begin{eqnarray*}
\bm{a}^2 &=& (0, 0, 0, \frac{\,1\,}{2}+\frac{\,1\,}{2}, \frac{\,1\,}{2}, 0, 0, 0, 0)\\
&=& (0,0,0,1,\frac{\,1\,}{2},0,0, 0, 0).
\end{eqnarray*}
The output of the layer is 
\[
\bm{q}^3 = H(\bm{a}^2) = (0,0,0,1,1,0,0,0, 0),
\]
which represents $s\wedge t$.
In the same manner, by letting $\bm{q}^{k+1} = H(\bm{a}^{k})$, we get
\[
\mathrm{hardmax}(\langle \bm{q}^3, \bm{k}_1\rangle,...,\langle \bm{q}^3, \bm{k}_9\rangle ) = (0, 0, 0, \frac{\,1\,}{2}, \frac{\,1\,}{2}, 0, 0, 0, 0),\]
\[
\bm{a}^3 = (0,0,0,0, 0,\frac{\,1\,}{2},0, \frac{\,1\,}{2},0),
\]
\[
\bm{q}^4 = H(\bm{a}^3 ) = (0,0, 0,0, 0,1,0, 1,0),
\]
which represents $u \wedge \top$,
\[
\mathrm{hardmax}(\langle \bm{q}^4, \bm{k}_1\rangle,..., \langle \bm{q}^4, \bm{k}_9\rangle) = (0,0,0, 0, 0, \frac{\,1\,}{2},0,\frac{\,1\,}{2},0),
\]
\[
\bm{a}^4 = (0,0, 0,0,0,0,0, 1,0),
\]
\[
\bm{q}^5 = H(\bm{a}^4 ) =(0,0, 0,0,0,0,0, 1,0),
\]
which represents $ \top$,
and the top-down derivation is successful. 

We give general definitions.
Let $\Pi = \{p_1, p_2, \ldots, p_{N} \}$ be the set of propositional symbols. 
For the convenience, we let $p_{N+1}=\top$, and $p_{N+2}=\bot$. 
We consider vectors in $\mathbb{R}^{N+2}$. 

For a definite clause $C$ we define a {\em head vector} $\bm{h}_C = (h_1, h_2,\ldots, h_{N+2})$
and a {\em body vector} $\bm{b}_C= (b_1, b_2,\ldots, b_{N+2})$ so that 
\begin{eqnarray*}
{h}_i &=&
\left \{
\begin{array}{ll}
1& \mbox{ if $\mathrm{head}(C)=p_i$,}\\
0& \mbox{ otherwise, }
\end{array}
\right .\\
{b}_i &=& 
\left \{
\begin{array}{ll}
1& \mbox{ if  $p_i$ appears in $\mathrm{body}(C)$,}\\
0& \mbox{ otherwise.}
\end{array}
\right .
\end{eqnarray*}
We regard a definite program $P$ as a sequence $C_1$, $C_2$, $\ldots$, $C_{N+2}$ of rules in it.

We define the head matrix $H_P$ and the body matrix $B_P$ as
\[
H_P= 
\left ( 
\begin{array}{c}
\bm{h}_1\\
\bm{h}_2\\
\vdots\\
\bm{h}_{N}\\
\bm{h}_\top\\
\bm{h}_\bot
\end{array}
\right ) \mbox{ and }
B_P= 
\left ( 
\begin{array}{c}
\bm{b}_1\\
\bm{b}_2\\
\vdots\\
\bm{b}_{N}\\
\bm{b}_\top\\
\bm{b}_\bot
\end{array}
\right ).
\]
It is clear that, for an SD-program $P$, $H_P$ is an identical matrix $I_{N+2}$ without loss of generality.
A conjunction for a conjunction, which is used as a query in top-down derivation is represented 
as a vector in the same manner of the body of a definite clause.
Then the one-step top-down derivation is represent with the self-attention function.
\begin{prop}
Assume that the self-attention 
function employs the  hardmax function and is followed by
 the dimension-wise Heaviside function $H$. 
Let $P$ be an SD-program and
$\bm{q}$ is a vector representing a query $Q$.
Then the output  $H(\mbox{\em Attention}(\bm{q}, H_P, B_P))$
represents
the query $Q'$ obtained by one-step top-down derivation of $Q$.
\end{prop}
\begin{proof}
For the head matrix $H_P=I_{N+2}$, $\langle \bm{q}, \bm{h}_i \rangle = 1$ if and only if 
 $p_i$ appears in $Q$. For all other $p_j$, 
$\langle \bm{q}, \bm{h}_j \rangle = 0$. 
This computation works as the selection of the definite clause which can be used 
for the top-down derivation.
Therefore the output of the self-attention function
\begin{eqnarray*}
\bm{a} &=& \mbox{hardmax}((\langle \bm{q}^k, \bm{h}_1\rangle,\ldots, \langle \bm{q}^k, \bm{h}_{N+2}\rangle))B_P\\
&=& (a_1,\ldots, a_{N+2})
\end{eqnarray*}
must satisfy that  $a_i= 0$ if and only if 
 $p_i$ does NOT appear in  $Q'$.
By applying $H(\,\cdot\,)$ we obtain the vector representing $Q$.
\end{proof}

\section{Implementing Bottom-Up Derivation with Self-Attention Network}
In this section 
we explain the method of implementing the bottom-up derivation of an SD-program proposed in~\cite{DBLP:conf/miwai/NguyenSSI18} with the self-attention netwroks.
For an SD-program $P$ a program matrix $M_P$ is defined in a slightly different manner of the body matrix $B_P$. 

First we have to note that they do not use the dimension of $\top$ or $\bot$ and therefore
vectors in ${\mathbb R}^N$ is used.

For a definite clause $C$ we define a {\em progam vector} $\bm{m}_C = (m_1, m_2, \ldots, m_{N})$
as
\begin{eqnarray*}
{m}_i &=& 
\left \{
\begin{array}{ll}
\frac{1}{M}& \mbox{ if $p_i$ appears in $\mathrm{body}(C)$,}\\
1 & \mbox{ if $\mathrm{body}(C)=\top$ and $\mathrm{head}(C)=p_i$,}\\
0& \mbox{ otherwise,}
\end{array}
\right . 
\end{eqnarray*}
where $M$ is the number of symbols in the body of the definite clause.
The {\em program matrix} $M_P$ is constructed by putting the program vectors vertically. 
For example, the program matrix of the SD-program presented in Section~2.2 is
\[
M_P = 
\begin{array}{c}
\begin{array}{c}
\\
\end{array}\\
\left ( 
\begin{array}{c}
\bm{m}_p\\
\bm{m}_q\\
\bm{m}_r\\
\bm{m}_s\\
\bm{m}_t\\
\bm{m}_u\\
\bm{m}_w
\end{array}
\right )
\end{array}
=
\begin{array}{c}
\begin{array}{ccccccc}
p\, & \,q\, &\, r \, &\, s \, &\, t \, & \,u \, &\,  w \\
\end{array}\\
\left (
\begin{array}{ccccccc}
0 & \frac{1}{\,2\,} &\frac{1}{\,2\,} & 0 &0 & 0 & 0 \\
0 & 0 & 0 & 1 & 0 & 0 & 0 \\
0 & 0 & 0 & \frac{1}{\,2\,} & \frac{1}{\,2\,} & 0 & 0\\
0 & 0 & 0 & 0 & 0 &  1 & 0 \\
0 & 0& 0 & 0 & 1 & 0 & 0 \\
0 & 0& 0 & 0 & 0 & 1 & 0 \\
0 & 0& 0 & 0 & 0 & 0 & 0 \\
\end{array}
\right )
\end{array}
\]

The vector $\bm{q}$ represents such an interpretation that the proposition $p$ is interpreted to be true
if and only if the dimension of $p$ of $\bm{q}$ is 1. 
Then the bottom-up derivation is represented with the attention function (\ref{attention})
by replacing the softmax function with the dimension-wise identity function and
letting $K=M_P$ and $V=H_P$. 

\section{Conclusion}
In this paper we show that top-down derivations from a conjunction of propositions as queries with
an SD-program are represented by self-attention networks followed by the FFNs representing the dimension-wise
Heaviside function. The self-attention network is modified by replacing the softmax function with the hardmax function.
We also show that bottom-up derivations from an interpretation with an SD-program are represented by the self-attention networks with the element-wise identity function as the substitution of the softmax function.
We believe that our results show that LLMs implicitly have the power of logical inference.

The previous work listed in Section~1 proposed that bottom-up derivations for several extensions of definite 
logic programs can be represented by operations of tensors.
We conjecture that our representation method could be modified for some of such extensions by using 
the multi-head attention networks~\cite{DBLP:conf/nips/VaswaniSPUJGKP17}.

The replacement of the softmax function with the hardmax is required by the point that we are based on the traditional binary-valued propositional logic.
The original definition of the self-attention networks employs the softmax function because LLMs works in probabilistic manners. 
Our future work includes some extensions of our discussion to probabilistic propositional logic so that we could show more potentials of practical uses of LLMs.

\section*{Acknowledgments}
This work is partly supported by ROIS NII Open Collaborative Research 2024-24S1202 and 
JSPS KAKENHI Grant Number JP20H0596.

\bibliographystyle{kr}
\bibliography{ThuyYamamoto241015}

\end{document}